\documentclass{article} 
\usepackage[preprint]{colm2026_conference}

\usepackage{algorithm}
\usepackage{algorithmicx}%
\usepackage{algpseudocode}%
\usepackage{amsmath}
\usepackage{amsfonts}
\usepackage{booktabs}
\usepackage{multirow}
\usepackage{tabularx}
\usepackage{makecell}
\usepackage{pifont}
\usepackage{newfloat}
\usepackage{listings}
\usepackage{adjustbox}

\usepackage{hyperref}
\usepackage{url}
\usepackage{wrapfig}
\usepackage{booktabs}
\usepackage{multirow}
\usepackage{wrapfig}
\usepackage{amssymb}
\usepackage{caption}
\usepackage{amsthm}
\usepackage{enumitem}
\theoremstyle{remark}

\usepackage{tabularx}
\usepackage{makecell}
\usepackage{pifont}
\usepackage{newfloat}
\usepackage{listings}
\usepackage[table]{xcolor}

\newtheorem{prop}{Proposition}

\usepackage{lineno}

\definecolor{darkblue}{rgb}{0, 0, 0.5}
\hypersetup{colorlinks=true, citecolor=darkblue, linkcolor=darkblue, urlcolor=darkblue}

\title{Distance Is All You Need: Radial Dispersion for \\Uncertainty Estimation in Large Language Models}

\author{
Manh Nguyen\thanks{Corresponding Author}, Sunil Gupta, and Hung Le\\
Applied Artificial Intelligence Initiative, Deakin University, Australia\\
\texttt{\{manh.nguyen, sunil.gupta, thai.le\}@deakin.edu.au}
}

\begin{document}

\ifcolmsubmission
\linenumbers
\fi

\maketitle

\begin{abstract}

Detecting uncertainty in large language models (LLMs) is essential for building reliable systems, yet many existing approaches are overly complex and depend on brittle semantic clustering or access to model internals. We introduce \textbf{Radial Dispersion Score (RDS)}, a simple, training-free, fully model-agnostic uncertainty metric that measures the radial dispersion of sampled generations in embedding space. Specifically, given $N$ sampled generations embedded on the unit hypersphere, RDS computes the total $\ell_1$ distance from the empirical centroid, i.e., the mean embedding, providing a direct geometric signal of semantic variability. A lightweight probability-weighted variant further incorporates the model's own token probabilities when available, outperforming nine recent state-of-the-art baselines. Moreover, RDS naturally extends to effective per-sample uncertainty estimates that complement probability- and consistency-based methods while remaining lightweight for practical use. Across four challenging free-form question-answering datasets and four LLMs, our metrics achieve state-of-the-art hallucination detection performance, while remaining robust and scalable with respect to sample size and embedding choice. These results highlight the practical value of RDS and its contribution toward improving the trustworthiness of LLMs. Code is publicly available at \url{https://github.com/manhitv/RDS}.

\end{abstract}

\section{Introduction}
Large language models (LLMs) exhibit remarkable reasoning and generation capabilities, yet they frequently produce \textit{hallucinations}, fluent but factually incorrect or unsubstantiated outputs~\citep{ji2023survey, huang2025survey}. Detecting when a language model is uncertain remains essential for building reliable, trustworthy systems. Among existing approaches, uncertainty estimation is one of the most effective tools for hallucination detection.

Predictive entropy is the information-theoretic gold standard for quantifying uncertainty~\citep{malinin2020uncertainty}, but exact computation is infeasible due to the exponential output space. As a result, recent work relies on Monte Carlo sampling. For example, semantic entropy~\citep{kuhn2023semantic,farquhar2024detecting} and its extensions~\citep{lin2023generating, nikitin2024kernel} cluster sampled responses in an external embedding space and compute entropy over semantic equivalence classes. While effective, these methods suffer from two critical drawbacks: (i) accurate semantic clustering is inherently brittle, i.e., a single response can belong to multiple plausible clusters, and (ii) by operating exclusively on the \emph{generation} space, they discard the LLM's own probability estimates and thus ignore a rich source of epistemic uncertainty directly available from the model itself. More recently, \citet{nguyen2025probabilities} utilised these probabilities by introducing PRO, an approximation of predictive entropy using only the top-$K$ generation probabilities derived from negative log-likelihood. Although effective on open-weight models, PRO remains model-specific, requires a calibration set to select the optimal $K$, and cannot be applied to black-box LLMs.

A parallel line of research approximates differential entropy through geometric properties of hidden representations. EigenScore~\citep{chen2024inside} uses the trace of the covariance matrix (equivalently, the sum of its eigenvalues) of LLM internal states as a lightweight proxy. While computationally attractive, EigenScore is fundamentally \emph{model-specific} because internal states are inaccessible for black-box LLMs and even for some open-weight models. In addition, it saturates in high-uncertainty regimes, mapping a wide range of highly uncertain cases to the same maximal value and thus reducing sensitivity, especially when the output distribution exhibits antipodal modes (see Figure~\ref{fig:examples}). 

To tackle these problems, we propose \textbf{Radial Dispersion Score (RDS)}, a simple, training-free, and model-agnostic uncertainty metric with a clean geometric interpretation. RDS avoids semantic clustering, does not access hidden states, and requires no calibration or model internals. Given $N$ sampled generations embedded on the unit hypersphere using an external encoder, RDS is defined as the total $\ell_1$ radial dispersion from the empirical centroid (the mean embedding). Intuitively, it directly quantifies the aggregate angular spread of output generations: the larger the sum of distances from the centroid, the greater the overall semantic dispersion.
Theoretically, RDS lower-bounds the trace of the embedding covariance matrix and thus preserves a monotonic connection to the differential entropy of the (unknown) continuous output distribution. Unlike EigenScore, it remains highly discriminative in high-entropy, opposing-cluster regimes.

We further introduce a probability-weighted variant, RDS$_w$, which incorporates LLM token-level probabilities when available, thereby emphasizing the actual structure of the model's predictive distribution. This weighting preserves meaningful directional distinctions and shifts the centroid toward high-probability generations, yielding a more faithful estimate of the representative embedding. Geometrically, RDS$_w$ is exactly the 1-Wasserstein distance between the empirical distribution of generations and its barycenter, providing a geometric foundation for its stability and sensitivity.
RDS applies uniformly to both black-box and open-weight models, while RDS$_w$ is available for models exposing token-level probabilities. 
In summary, our contributions are twofold:
\begin{itemize}
    \item We introduce RDS and RDS$_w$, a family of simple, parameter-free, model-agnostic uncertainty estimators grounded in radial dispersion geometry, supported by formal theoretical analyses.
    \item Across four free-form QA datasets, four diverse LLMs, our metrics deliver state-of-the-art performance and demonstrate strong robustness and scalability.
\end{itemize}

\section{Related Work}

Uncertainty estimation methods for LLMs can be grouped into three main families. (1) \textit{Probability-based methods} include (average) negative log-likelihood remain surprisingly competitive~\citep{guerreiro2022looking,manakul2023selfcheckgpt}. PRO~\citep{nguyen2025probabilities} recently improved upon these by approximating predictive entropy using only the top-$K$ generation probabilities and an adaptive threshold to filter noisy samples. While effective on open-weight models, these approaches cannot be used with black-box APIs and often require calibration data. (2) \textit{Semantic entropy and its extensions}~\citep{kuhn2023semantic,farquhar2024detecting,lin2023generating,nikitin2024kernel,qiu2024semantic} cluster sampled responses in an external embedding space and compute entropy over clusters' entropy (or density). Despite strong performance, they inherit the brittleness of clustering and discard the LLM's own probability signal. (3) \textit{Geometric methods} such as EigenScore~\citep{chen2024inside}, use the covariance trace as a fast entropy proxy but require internal states and saturate in high-uncertainty regimes such as antipodal distributions. Geometric Volume~\citep{phillips2025geometric} is model-agnostic via convex-hull of archetypes but computationally heavy due to non-convex archetypal analysis and poor scaling of hull computation. Parallel work explores multi-model uncertainty via Jensen–Shannon divergence~\citep{kruse2025simple}, and hidden-state Monte-Carlo sampling~\citep{flue}.
Our method belongs to none of these families: it is geometric yet model-agnostic, leverages generation probabilities only optionally, and avoids clustering entirely, and is provably more sensitive than trace-based metrics.

\section{Preliminaries}\label{sec:prelim}

\paragraph{From Discrete to Continuous Entropy}

The information-theoretic gold standard for predictive uncertainty is the conditional entropy of the output distribution:
\begin{equation}
  \label{eq:pe_origin}
  H(Y|x) = -\sum_{y} p(y|x)\log p(y|x)
\end{equation}
where \(Y\) is the (discrete) output random variable, \(x\) is the input. A low predictive entropy indicates a sharply concentrated output distribution, whereas a high value indicates more diverse outputs.

Differential entropy is the continuous analogue of this quantity, obtained by replacing discrete probabilities with a density \(f(y|x)\). In the context of LLMs, such a density is induced by embedding sampled completions using either internal hidden states or an external encoder. 
Approximating the conditional distribution of embedding vectors as Gaussian \(Y|x\!\sim\!\mathcal{N}(\mu,\Sigma)\) yields a closed-form differential entropy \citep{zhouyin2021understanding}:
\begin{align}
    H_{\mathrm{de}}(Y|x)
    &= \tfrac{1}{2} \log \det(\Sigma) + \tfrac{d}{2}(\log 2\pi + 1) \\
    &= \tfrac{1}{2}\sum_{i=1}^d \log \lambda_i + C,
    \label{eq:diff_entropy}
\end{align}
where \(\lambda_i\) are the eigenvalues of \(\Sigma\), \(d\) is the embedding dimension, and \(C\) is an additive constant.

\paragraph{EigenScore as a Practical Proxy}
While differential entropy provides a principled measure, computing it requires estimating the full covariance matrix eigenspectrum. \citet{chen2024inside} introduces a computationally simple surrogate for differential entropy: the \emph{EigenScore}, defined as the trace of the covariance of the hidden representations:

\begin{align}
    \label{eq:eigenscore}
    \operatorname{EigenScore}(X) &= \operatorname{tr}(\Sigma) \\
    &= \frac{1}{N} \sum_{i=1}^N \left\| \mathbf{u}_i -\mathbf{\bar{u}} \right\|_2^2 \\
    &=\sum_{i=1}^N \lambda_i  \propto H_{\mathrm{de}}(X),
\end{align}
where \(\mathbf{u}_i\) denotes the embedding of the $i$-th sampled completion and \(\mathbf{\bar{u}} = \frac{1}{N}\sum_{i=1}^N \mathbf{u}_i\) is the sample mean embedding. EigenScore equals the average squared \(\ell_2\) distance from the centroid: it approaches \(0\) when embeddings collapse (low uncertainty) and increases with isotropic spread (high uncertainty).

\section{Methodology}\label{sec:method}

\begin{figure*}[t]
    \centering
    \includegraphics[width=\textwidth]{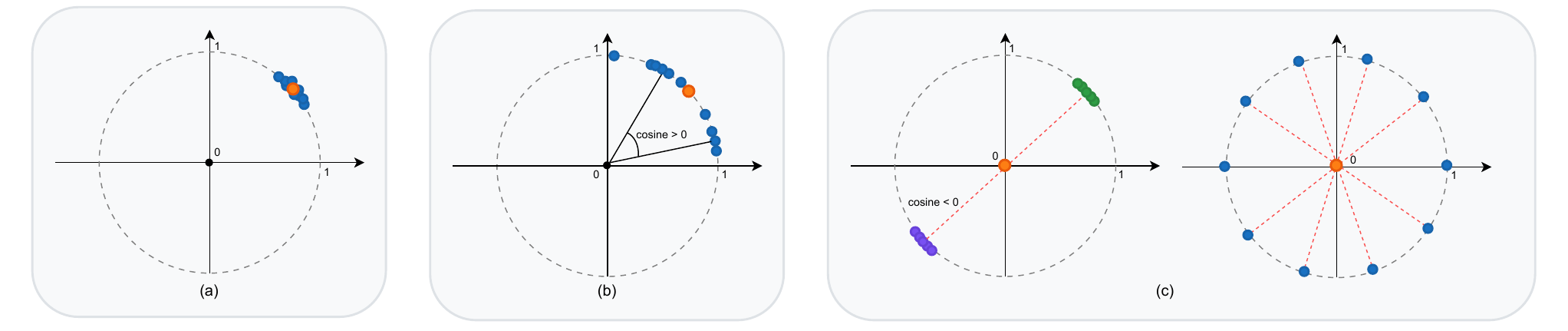}
    \caption{
    RDS vs. EigenEmbed across three uncertainty regimes illustrated using ten unit-norm 2D embeddings. Orange circles denote the empirical centroid (mean vector).
    (a)~Low uncertainty: all generations collapse into a tight semantic cluster; RDS $\approx$ EigenEmbed $\approx 0$.
    (b)~Hemispheric spread: embeddings are broadly dispersed but remain positively aligned with the mean (all cosines with the centroid > 0); Both metrics indicate moderate-to-high uncertainty (RDS $\approx \sqrt{N}$, EigenEmbed $\approx$ 0.8--0.9).
    (c)~Opposing-cluster regime: two or more tight clusters with negative inter-cluster cosines cancel the mean vector toward zero, causing EigenEmbed to saturate near 1 despite growing structural disagreement, whereas RDS scales with the magnitude of these opposing directions ($\ge \sqrt{N}$), making it substantially more sensitive and diagnostically useful for detecting high-uncertainty generations.
    }
    \label{fig:examples}
\end{figure*}

\subsection{Uncertainty Estimation via Radial Dispersion Score}\label{subsec:rds}

To estimate the uncertainty of an LLM's output given a prompt \(x\), we leverage the geometry of semantic embeddings sampled from the model's generations. Specifically, we generate \(N\) ($N>1$) sequences \(\{y_1, y_2, \dots, y_N\}\) conditioned on \(x\) using multinomial sampling. Each generation is embedded via an encoder $\mathbf{E}$ and then explicitly $\ell_2$-normalized to obtain unit-norm vectors $\mathbf{u}_i \in \mathbb{R}^d$:

\begin{equation}
    \mathbf{u}_i = \frac{\mathbf{E}(y_i)}{\|\mathbf{E}(y_i)\|_2}, 
    \qquad \|\mathbf{u}_i\|_2 = 1.
\end{equation}

These embeddings represent points on the unit hypersphere, where proximity reflects semantic similarity among output generations. High uncertainty manifests as dispersed embeddings (diverse plausible outputs), while low uncertainty yields clustered embeddings \citep{kuhnsemantic}. We quantify this dispersion simply using the \(\ell_1\)-norm dispersion from the centroid, termed Radial Dispersion Score (RDS), which is defined as:

\begin{equation}
    \label{eq:RDS}
    \operatorname{RDS}(x)
    = \sum_{i=1}^N 
        \left\|
            \mathbf{u}_i - 
            \mathbf{\bar{u}}
        \right\|_1.
\end{equation}
where \(\mathbf{\bar{u}} = \frac{1}{N}\sum_{i=1}^N \mathbf{u}_i\) denotes the empirical centroid of the sampled embeddings. Intuitively, RDS measures the total radial dispersion of each embedding from the centroid, with larger values indicating greater dispersion and, consequently, higher uncertainty, vice versa. 

We choose the $\ell_1$-norm over the $\ell_2$-norm for two reasons: (1) it provides a tighter upper bound on variance-based measures (as formalized in Proposition~\ref{prop:RDS_bound}, Section~\ref{sec:eigenscore}), and (2) it responds more strongly to large deviations in the embedding space~\citep{hastie2005elements}, making extreme dispersion easier to detect.

\paragraph{Probability-Weighted Variant}

Geometric dispersion alone overlooks variation in generation likelihoods. Prompt ambiguity, task difficulty, or knowledge gaps can make some outputs far more probable than others \citep{kuhn2023semantic, hou2024decomposing}, while generation probabilities have been shown to correlate with correctness \citep{kadavath2022language}. To incorporate this, we define a \textit{probability-weighted} variant that emphasizes dispersion among high-probability outputs while reducing the impact of low-probability, potentially noisy generations \citep{nguyen2025probabilities}. Let \(p_i = p(y_i|x)/\sum_j p(y_j|x)\) denote the normalized generation likelihood of each sequence, the weighted RDS at the prompt level is then defined as:

\begin{equation}
    \label{eq:wRDS}
    \operatorname{RDS}_w(x)
    = \sum_{i=1}^N p_i \left\|
        \mathbf{u}_i - \bar{\mathbf{u}}_w
    \right\|_1,
    \,
    \bar{\mathbf{u}}_w = \sum_{i=1}^N p_i \mathbf{u}_i.
\end{equation}

Here, the probability of each generation $p(y_i|x)$ is estimated empirically by Average Negative Log-Likelihood (ANLL) \citep{manakul2023selfcheckgpt}, which is widely used in prior work~\citep{guerreiro2022looking, manakul2023selfcheckgpt, nguyen2025probabilities}.

Beyond its geometric simplicity, \textbf{RDS is inherently robust and broadly applicable}: it relies solely on generated outputs rather than LLM internal states, operates in a lower-dimensional embedding space, and avoids model-specific architectural assumptions. This makes RDS easy to compute, lightweight in practice, and compatible with any black-box LLM. Weighted RDS further refines this signal when generation probabilities are available (e.g., grey-box or open-weight models), yielding a more faithful estimate of uncertainty by prioritizing dispersion among high-probability outputs. 

\subsection{Theoretical Analysis}\label{sec:eigenscore}

\subsubsection{RDS Connection to EigenScore}
To begin, we introduce \textbf{EigenEmbed}, defined as EigenScore computed from an \textit{external encoder} $\mathbf{E}$ rather than LLM internal hidden states. This allows fair comparison with RDS, which also operates on external embeddings. Figure~\ref{fig:examples} illustrates a 2D geometric intuition comparing RDS and EigenEmbed using ten unit-norm embeddings. In the coherent (a) and hemispheric-spread (b) regimes, both metrics agree. The most critical divergence occurs in the opposing-cluster regime (Figure~\ref{fig:examples}c), a prevalent failure mode in which LLMs produce overconfident yet contradictory generations. Here, multiple tight semantic clusters lie in roughly opposite directions on the hypersphere, resulting in a strong cancellation of the mean vector. As a result, EigenEmbed rapidly saturates and loses discriminative power, whereas RDS continues to scale with the true angular separation. We formalize this discrepancy below.

\begin{prop}
\label{prop:RDS_bound}
Let \(\{\mathbf{u}_i\}_{i=1}^N \subset \mathbb{R}^d\) be unit-norm embeddings (\(\|\mathbf{u}_i\|_2 = 1\)). Let the centered embeddings be \(\mathbf{v}_i = \mathbf{u}_i - \bar{\mathbf{u}}\). Then:

\begin{enumerate}[label={(\arabic*)}]
    \item $\operatorname{EigenEmbed} = \frac{1}{N}\sum_{i=1}^N \left\| \mathbf{v}_i\right\|_2^2 \in [0, 1]$.
    \item $\operatorname{RDS} \geq \operatorname{EigenEmbed}$.
\end{enumerate}
Equality in (2) holds if and only if all \(\mathbf{u}_i\) are identical. The gap becomes larger as $\bar{\mathbf{u}} \to \mathbf{0}$.
\end{prop}
\begin{proof}
See Appendix~\ref{app:proof_RDS_bound}.    
\end{proof}

This parameter-free lower bound is tight for any \(d \geq 1\) and $N \geq 2$. 

\paragraph{Extremal Case: \(\bar{\mathbf{u}} = \mathbf{0}\).} When the centroid is at zero, the embeddings form a perfectly \emph{balanced} configuration around the origin (Figure~\ref{fig:examples}c), giving \(\mathbf{v}_i = \mathbf{u}_i\). EigenEmbed then attains its maximum:
\[
\operatorname{EigenEmbed} = \frac{1}{N}\sum_{i=1}^N \|\mathbf{u}_i\|_2^2 = 1.
\]
For RDS, non-negativity of $\|\mathbf{v}_i\|_1 \, \forall i$ yields 
\[
\operatorname{RDS} = \sum_{i=1}^N \|\mathbf{v}_i\|_1
    \;\ge\;\sqrt{\sum_{i=1}^N \|\mathbf{v}_i\|_2^2}
    = \sqrt{N}.
\]

Thus, the two metrics diverge sharply in this regime, with RDS remaining responsive to the extent of semantic dispersion. Moreover, a zero centroid imposes further structural constraints on the embeddings as follows.

\begin{prop}
\label{prop:cosine_zero_centroid}
If the embeddings satisfy \(\bar{\mathbf{u}} = \mathbf{0}\), then the average pairwise cosine similarity is
\begin{equation}
    \frac{1}{N(N-1)} 
    \sum_{i \neq j} \mathbf{u}_i^\top \mathbf{u}_j
    = -\frac{1}{N-1}
    < 0.
\end{equation}
Consequently, for every \(i\), there exists at least one \(j \neq i\) such that 
\(\mathbf{u}_i^\top \mathbf{u}_j < 0\).
\end{prop}

\begin{proof}
See Appendix~\ref{app:proof_zero_centroid}.    
\end{proof}

This negative-average structure is a \emph{geometric signature of semantic disagreement}: sampled generations are not merely diverse but actively opposed, forming antipodal or multi-cluster cancellation regimes, an explicit signature of semantic diversity and uncertainty. Such geometric configurations highlight regimes of maximal uncertainty, independent of the specific choice of dispersion metric.



\subsubsection{RDS$_w$ Connection to Optimal Transport}

The weighted RDS$_w$ coincides exactly with the $1$-Wasserstein distance (Earth Mover’s Distance) between the discrete probability measure
\[
\mu = \sum_{i=1}^N p_i \, \delta_{\mathbf{u}_i} \in \mathcal{P}(\mathbb{R}^d)
\]
and the Dirac measure at its own barycenter
\[
\nu = \delta_{\bar{\mathbf{u}}_w}, \qquad \bar{\mathbf{u}}_w = \sum_{i=1}^N p_i \, \mathbf{u}_i,
\]
when the ground cost is $c(\mathbf{a},\mathbf{b}) = \|\mathbf{a} - \mathbf{b}\|_1$.

Because the target measure \(\nu\) consists of a single Dirac mass, every feasible coupling must send the entire mass \(p_i\) located at \(\mathbf{u}_i\) directly to \(\bar{\mathbf{u}}_w\). Hence, the coupling set \(\Pi(\mu,\nu)\) contains a single deterministic plan, yielding the closed form~\citep{santambrogio2015optimal, peyre2019computational}:
\[
W_1(\mu,\nu)
= \sum_{i=1}^N p_i\,\|\mathbf{u}_i - \bar{\mathbf{u}}_w\|_1
= \operatorname{RDS}_w(x).
\]

This closed-form expression makes RDS$_w$ computationally trivial (no optimization required), while inheriting the geometric meaning and theoretical properties of the $1$-Wasserstein distance.

\section{Experimental Results}\label{sec:experiment}

\subsection{Experiment Setup}\label{sec:experiment-setup}


\paragraph{Datasets} We use four established benchmarks covering scientific QA and mathematical reasoning: (1) Scientific QA: SciQ \citep{welbl2017crowdsourcing}, and GPQA Diamond (GPQA) \citep{rein2024gpqa}, (2) Mathematical reasoning: Arithematics \citep{brown2020language}, and SVAMP \citep{patel2021nlp}. All benchmarks are experimented over full test sets. We provide the links to these datasets in Table~\ref{tab:data}, Appendix~\ref{app:data}.

\paragraph{Models} We evaluate on four popular instruction-tuned open-weight models from distinct families: Falcon3-7B~\citep{almazrouei2023falcon}, Gemma2-9B~\citep{team2024gemma}, Llama3.1-8B, and Llama3.2-3B~\citep{grattafiori2024llama}. We refer to them as Falcon3, Gemma2, Llama3.1, and Llama3.2.

\paragraph{Baselines} 
We compare our method against nine established uncertainty estimators: (1) Average Log Likelihood (ANLL) \citep{guerreiro2022looking}, (2) Negative Log Likelihood (NLL) \citep{aichberger2024rethinking}, (3) PRO \citep{nguyen2025probabilities}, (4) Semantic Entropy (SE) \citep{kuhn2023semantic}, (5) Degree (Deg) \citep{lin2023generating}, (6) Semantic Density (SD) \citep{qiu2024semantic}, (7) Self Consistency (SC) \citep{wang2022self}, (8) EigenScore (ES) \citep{chen2024inside}, and (9) EigenEmbed (EE), which is EigenScore computed on external embeddings. 

\begin{table*}[ht]
  \caption{
    AUC performance comparison across datasets and models. All values are reported in percentages. Best scores are bolded, second-best values are underlined. RDS$_{\ell_2}$ denotes our base metric computed using $\ell_2$-norm (Eq.~\eqref{eq:RDS}), included for reference. Our metrics (\textbf{RDS$_{\ell_2}$}, \textbf{RDS}, and \textbf{RDS{$_w$}}) are highlighted with shaded backgrounds. ES is unavailable (--) on Gemma2 because the Gemma family does not expose hidden states.
  }\label{tab:main_results}
  \centering
  \resizebox{\textwidth}{!}{
  \begin{tabular}{cccccccccccccc}
    \toprule
    \textbf{Dataset} & \textbf{Model} &
    \textbf{ANLL} & \textbf{NLL} & \textbf{PRO} & \textbf{SE} & \textbf{Deg} & \textbf{SD} & \textbf{SC} &
    \textbf{ES} & \textbf{EE} & \textbf{RDS$_{\ell_2}$} & \textbf{RDS} & \textbf{RDS$_w$} \\
    \midrule

    \multirow{4}{*}{GPQA}
    & Falcon3
      & \textbf{72.5} & 62.2 & 60.2 & 50.4 & 66.3 & 65.3 & 51.6 & 64.9 & 65.6 & \cellcolor{yellow!20}66.5 & \cellcolor{yellow!20}\underline{67.5} & \cellcolor{yellow!20}67.0 \\
    & Gemma2
      & 62.6 & 63.6 & \textbf{65.2} & 50.5 & 63.2 & 61.9 & 51.6 & -- & 63.0 & \cellcolor{yellow!20}62.5 & \cellcolor{yellow!20}63.0 & \cellcolor{yellow!20}\underline{63.7} \\
    & Llama3.1
      & 64.1 & 64.1 & 59.4 & 53.3 & 59.0 & 58.2 & 51.6 & 63.3 & 62.4 & \underline{64.7}\cellcolor{yellow!20} & \cellcolor{yellow!20}64.4 & \cellcolor{yellow!20}\textbf{66.0} \\
    & Llama3.2
      & 57.2 & 57.2 & 63.8 & 48.2 & 56.4 & 54.9 & 59.8 & 56.8 & 63.9 & \underline{64.5}\cellcolor{yellow!20} & \cellcolor{yellow!20}64.3 & \cellcolor{yellow!20}\textbf{67.1} \\

    \midrule

    \multirow{4}{*}{SciQ}
    & Falcon3
      & 60.0 & 57.5 & 62.1 & 57.0 & 70.2 & 69.4 & 69.1 & 62.4 &
        73.2 & \cellcolor{yellow!20}72.9 & \cellcolor{yellow!20}\underline{73.4} & \cellcolor{yellow!20}\textbf{74.2} \\
    & Gemma2
      & 59.9 & 59.3 & 68.2 & 59.0 & 72.5 & 72.4 & 74.0 & -- &
        74.1 & \cellcolor{yellow!20}75.1 & \cellcolor{yellow!20}\textbf{75.4} & \cellcolor{yellow!20}\underline{75.2} \\
    & Llama3.1
      & 64.4 & 64.4 & 57.6 & 63.7 & 75.2 & 73.8 & 76.5 & 56.4 &
        77.0 & \cellcolor{yellow!20}77.8 & \cellcolor{yellow!20}\underline{78.4} & \cellcolor{yellow!20}\textbf{78.8} \\
    & Llama3.2
      & 64.2 & 64.2 & 54.5 & 65.1 & 72.5 & 71.0 & 73.2 & 59.1 &
        73.2 & \cellcolor{yellow!20}75.0 & \cellcolor{yellow!20}\underline{75.1} & \cellcolor{yellow!20}\textbf{75.3} \\

    \midrule

    \multirow{4}{*}{Arithematics}
    & Falcon3
      & 70.0 & 76.7 & 77.0 & 83.3 & \textbf{89.9} & \underline{89.7} & 85.4 & 85.7 &
        83.6 & 83.3\cellcolor{yellow!20} & \cellcolor{yellow!20}85.3 & \cellcolor{yellow!20}86.6 \\
    & Gemma2
      & 49.6 & 49.6 & 56.4 & 83.4 & 84.7 & 84.1 & 83.2 & -- &
        82.6 & 84.7\cellcolor{yellow!20} & \cellcolor{yellow!20} \underline{85.1} & \cellcolor{yellow!20}\textbf{86.3} \\
    & Llama3.1
      & 71.3 & 71.3 & 56.8 & 84.7 & 87.5 & 87.9 & 86.3 & 64.4 &
        83.6 & 85.4\cellcolor{yellow!20} & \cellcolor{yellow!20}\textbf{88.7} & \cellcolor{yellow!20}\underline{88.4} \\
    & Llama3.2
      & 71.2 & 71.2 & 67.7 & 87.9 & 87.8 & 87.9 & \underline{88.8} & 58.4 &
        87.6 & 87.2\cellcolor{yellow!20} & \cellcolor{yellow!20}87.9 & \cellcolor{yellow!20}\textbf{89.0} \\

    \midrule

    \multirow{4}{*}{SVAMP}
    & Falcon3
      & 91.7 & 91.7 & 92.3 & 89.2 & 90.9 & 93.0 & 94.4 & 66.6 &
        93.4 & 94.0\cellcolor{yellow!20} & \cellcolor{yellow!20}\underline{94.7} & \cellcolor{yellow!20}\textbf{95.1} \\
    & Gemma2
      & 53.3 & 47.5 & 69.1 & 75.3 & 83.2 & 83.8 & \underline{84.9} & -- &
        82.8 & 82.9\cellcolor{yellow!20} & \cellcolor{yellow!20}83.7 & \cellcolor{yellow!20}\textbf{85.0} \\
    & Llama3.1
      & 57.9 & 57.9 & 68.9 & 84.3 & 86.7 & 87.8 & \textbf{92.4} & 64.4 &
        90.4 & 90.2\cellcolor{yellow!20} & \cellcolor{yellow!20}90.6 & \cellcolor{yellow!20}\underline{91.4} \\
    & Llama3.2
      & 62.1 & 62.1 & 57.7 & 80.3 & 80.0 & 79.3 & \underline{86.2} & 60.8 &
        84.8 & 84.4\cellcolor{yellow!20} & \cellcolor{yellow!20}84.6 & \cellcolor{yellow!20}\textbf{86.4} \\

    \midrule
    \textbf{Average} & &
      64.5 & 63.8 & 64.8 & 69.7 &
      76.6 & 76.3 & 75.6 &
      63.6 & 77.6 & 78.2\cellcolor{yellow!20} & \cellcolor{yellow!20}\underline{78.9} & \cellcolor{yellow!20}\textbf{79.7} \\
    \textbf{Best Count} & &  1 & 0 & 1 & 0 & 1 & 0 & 1 & 0 & 0 & 0\cellcolor{yellow!20} & \cellcolor{yellow!20}2 & \cellcolor{yellow!20}\textbf{10} \\
    \bottomrule
  \end{tabular}}
\end{table*}

\paragraph{Evaluation Protocol} Following \citet{kuhn2023semantic,qiu2024semantic,nguyen2025probabilities}, we measure the ability of each uncertainty score to separate correct from incorrect greedy generations using Area Under the ROC Curve (AUC, $\%$). 
A generation is deemed correct if it satisfies an exact match on reasoning tasks (Arithmetics, SVAMP) or achieves an ROUGE-L F1 score \citep{lin2004rouge} greater than 0.3 on QA tasks (SciQ, GPQA), consistent with prior work.

\paragraph{Sampling and Implementation} We sample $N{=}10$ completions per question using multinomial sampling at temperature $\tau{=}1$ via vLLM~\citep{kwon2023efficient}. All sampling-based baselines use these same $N{=}10$ generations (except ANLL/NLL, which use only the greedy output). For self-consistency, uncertainty is computed as $1 - {\text{count of majority answer}}/{N}$. For all baselines, we use the suggested hyperparameters following their original setup (PRO: $\alpha{=}0.4$, CE: $p{=}0.3$). Our RDS variants and EigenEmbed use the widely adopted \texttt{all-MiniLM-L6-v2} sentence transformer~\citep{reimers2019sentence}. All experiments run on a single H100-80GB GPU.

\subsection{Main Results}\label{sec:result}

Table~\ref{tab:main_results} summarizes AUC performance across all 16 model--dataset pairs.

\paragraph{Our methods consistently outperform all baselines across all settings.} On average, RDS$_w$ achieves the highest AUC of 79.7\%, followed by RDS and RDS$_{\ell_2}$ at 78.9\% and 78.2\%, surpassing the next best (EE) by a clear margin of 2.1, 1.3, and 0.6\%, respectively. 
The stable superiority of $\ell_1$–based variants (RDS and RDS$_w$) over their $\ell_2$ counterpart (RDS$_{\ell_2}$), especially on Arithmetics, further validates our design choice, fully consistent with the rationale in Section~\ref{subsec:rds}.
Notably, RDS$_w$ ranks first in 10 out of 16 settings, demonstrating remarkable robustness across diverse tasks and model architectures.

The superiority of our scores is particularly pronounced on mathematical datasets (Arithmetic and SVAMP), where the gap over EE widens to 3--5\% in several cases. This aligns with our theoretical analysis in Section~\ref{sec:eigenscore}: when generated answers exhibit high lexical diversity (common in free-form math solutions, where responses differing by even a single character or magnitude, e.g., ``1'' vs.\ ``10'', are semantically distant), the proposed RDS metric better captures the underlying semantic clustering than eigenvector-based methods.

Self-Consistency (SC) performs strongly on math problems (frequently ranking in the top-3 on Arithmetics and SVAMP), but its effectiveness drops considerably on datasets where exact-match evaluation is not applicable (GPQA and SciQ). This confirms that SC remains highly dependent on the availability of a verifiable, correct answer. Other baselines, such as SD and Deg show competitive results on specific tasks (mostly QA), but show limited performance on mathematics datasets. Simple probability-based methods (NLL, ALL) generally lag behind, underscoring the importance of modeling response similarity in the embedding space.
Overall, the proposed RDS and weighted variant RDS$_w$ maintain robustness across model families and dataset characteristics.

\subsection{Ablation Study}\label{sec:ablation}

In this section, we investigate the effects of hyperparameters, including the number of sampling responses $N$ and choice of embedding model $\mathbf{E}$. To reduce computational overhead, we use Falcon3-7B for experiments over two representative datasets: GPQA and SVAMP.

\begin{figure}[t]
    \centering
    \includegraphics[width=\linewidth]{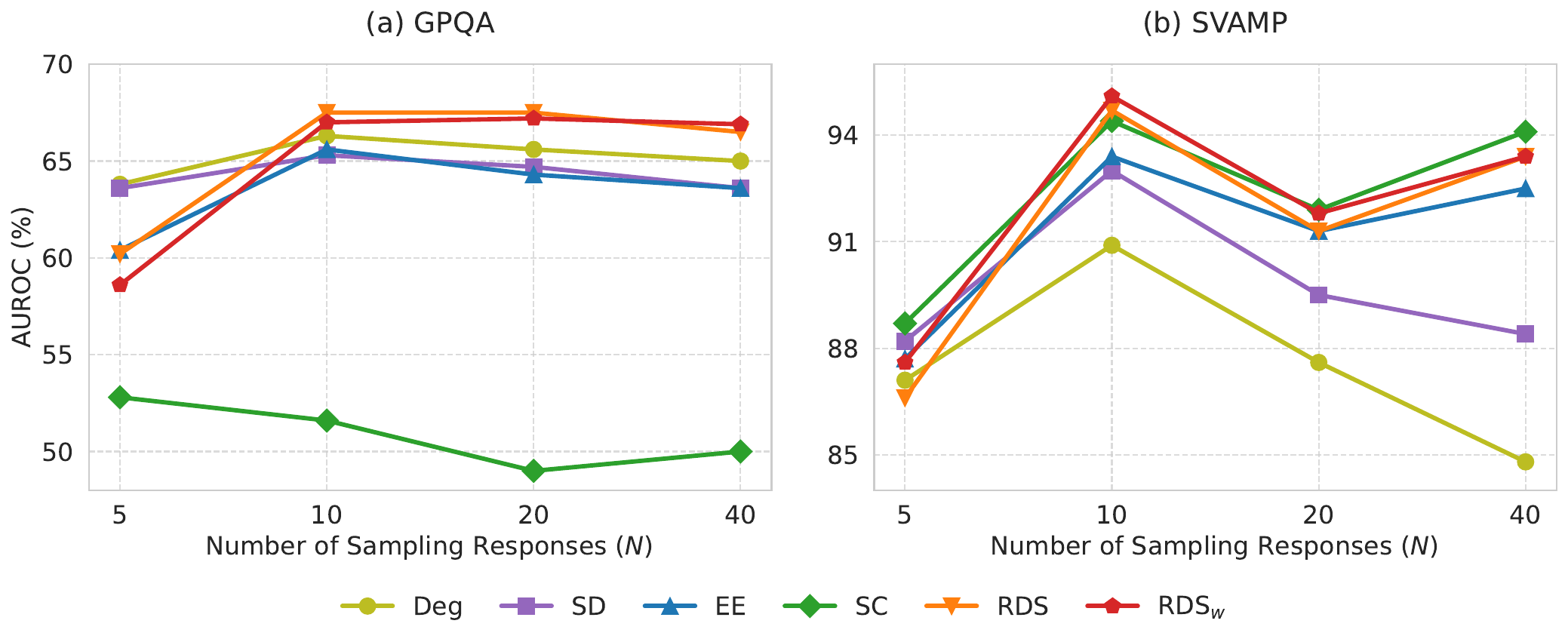}
    \caption{
    Ablation on the number of sampled responses $N$ for hallucination detection across datasets. Only top baselines are selected for illustration. Detailed results of all methods are provided in Appendix~\ref{app:sample-details}.
    }
    \label{fig:n-samples}
\end{figure}


\begin{figure}[t]
    \centering
    \includegraphics[width=\linewidth]{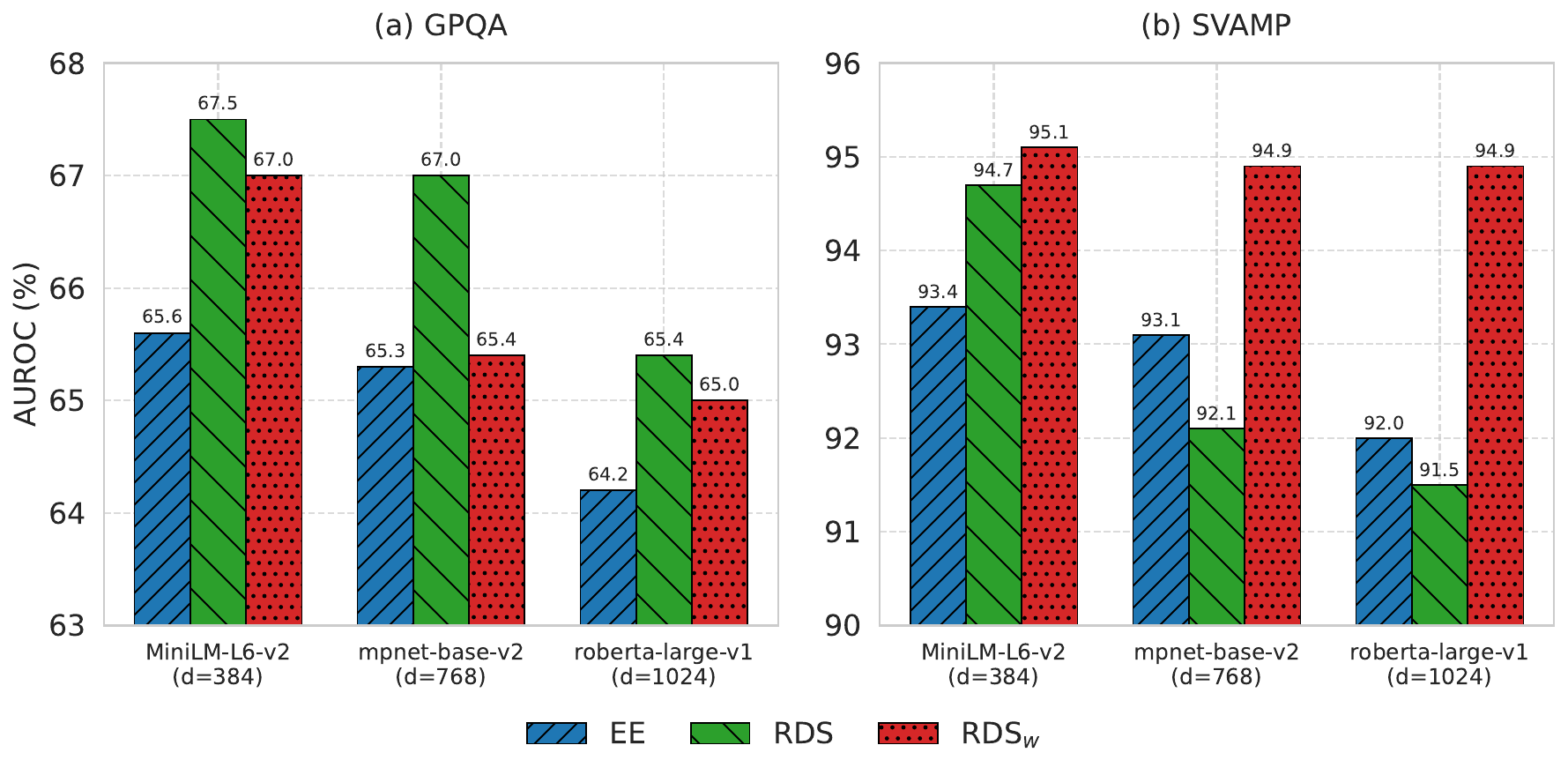}
    \caption{
    Effect of the sentence embedding model on hallucination detection across (a) GPQA and (b) SVAMP using $N{=}10$ .
    }
    \label{fig:ablation-embedding}
\end{figure}

\paragraph{Effect of the Number of Sampling Responses}

We vary $N \in \{5, 10, 20, 40\}$ using Falcon3-7B on GPQA and SVAMP (Figures~\ref{fig:n-samples}).
On GPQA, most semantic baselines (Deg, SD, EE) peak at $N{=}10$ and degrade as $N$ increases, suggesting that additional low-probability samples introduce noise. Self-consistency (SC) performs near random (approximately 50\%) at large $N$, as majority voting fails without a clearly dominant answer. In contrast, RDS-based methods remain robust: RDS$_w$ shows essentially no degradation and achieves the highest AUC even at $N{=}40$. On SVAMP, similar trends are observed, though variance is smaller and SC remains stable at high $N$, as correct mathematical solutions are often repeated verbatim. 
Overall, RDS-based metrics are uniquely tolerant to noisy generations, making them well-suited to large sampling budgets.

\paragraph{Effect Of Embedding Models}

In Figure~\ref{fig:ablation-embedding}, we compare three sentence transformers of increasing capacity: \texttt{all-miniLM-L6-v2} (384-d, our default), \texttt{all-mpnet-base-v2} (768-d), and \texttt{all-roberta-large-v1} (1024-d). On GPQA, all methods (EE, RDS, and RDS$_w$) degrade slightly (1.4–2.1\%) with larger encoders, consistent with observations that larger contrastively trained sentence transformers can exhibit embedding space anisotropy and representation biases, where certain semantic dimensions dominate the embeddings, leading to suboptimal performance on specialized QA domains~\citep{nikolaev-pado-2023-representation}. In contrast, RDS$_w$ remains highly stable on SVAMP, with AUC varying by at most $\pm 0.2$\%. Overall, these results indicate that incorporating generation probabilities reduces sensitivity to encoder choice while preserving most performance of the lightweight default model.

\section{Conclusion}

We propose Radial Dispersion Score (RDS), a simple, training-free, model-agnostic uncertainty estimator that measures radial dispersion of sampled generations, with an optional probability-weighted variant. Across four QA datasets and four LLMs, RDS achieves state-of-the-art hallucination detection performance, demonstrating robustness and scalability with respect to sample size and embedding choice.
Future work could extend RDS to broader applications that involve comparing and aggregating multiple candidate outputs, such as ranking, selection, and consensus estimation, enabling more general uncertainty-aware decision making.

\bibliography{colm2026_conference}
\bibliographystyle{colm2026_conference}

\clearpage
\appendix
\section{Appendix}\label{sec:app}

\subsection{Proof of Proposition \ref{prop:RDS_bound}}\label{app:proof_RDS_bound}

\begin{proof}

Let \(\mathbf{v}_i = \mathbf{u}_i - \bar{\mathbf{u}}\) denote the centered embeddings.

\paragraph{Step 1: EigenEmbed bounds}
We first compute the squared \(\ell_2\) norms of the centered embeddings:
\begin{align}
    \sum_{i=1}^N \|\mathbf{v}_i\|_2^2 &= \sum_{i=1}^N \|\mathbf{u}_i - \bar{\mathbf{u}}\|_2^2 \\
    &= \sum_{i=1}^N \big(\|\mathbf{u}_i\|_2^2 - 2 \mathbf{u}_i^\top \bar{\mathbf{u}} + \|\bar{\mathbf{u}}\|_2^2\big).
\end{align}

Since \(\|\mathbf{u}_i\|_2^2=1\) and \(\sum_i \mathbf{u}_i = N\bar{\mathbf{u}}\), the middle term simplifies, yielding
\[
\sum_{i=1}^N \|\mathbf{v}_i\|_2^2 = N - 2 N \|\bar{\mathbf{u}}\|_2^2 + N \|\bar{\mathbf{u}}\|_2^2 = N\big(1-\|\bar{\mathbf{u}}\|_2^2\big).
\]

Hence, by definition,
\[
\operatorname{EigenEmbed} = \frac{1}{N} \sum_{i=1}^N \|\mathbf{v}_i\|_2^2 = 1 - \|\bar{\mathbf{u}}\|_2^2 \in [0,1].
\]

\paragraph{Step 2: RDS lower bound}

By definition, the $\ell_1$ norm of each centered embedding is non-negative: $\|\mathbf{v}_i\|_1 \geq 0$ for all $i$, which implies
\begin{equation}
    \sum_{i=1}^N \|\mathbf{v}_i\|_1 \geq \sqrt{\sum_{i=1}^N \|\mathbf{v}_i\|_2^2}. \label{eq:l1-l2_origin_app}
\end{equation}

We can rewrite the right-hand side in terms of $\operatorname{EigenEmbed}$:
\begin{align}
    \operatorname{RDS} &= \sum_{i=1}^N \|\mathbf{v}_i\|_1 \\
    &\geq \sqrt{\sum_{i=1}^N \|\mathbf{v}_i\|_2^2} \label{eq:l2-eigen_app} \\
    &= \sqrt{N \cdot \operatorname{EigenEmbed}} \label{eq:eigen_score_app} \\
    &\geq \sqrt{N} \cdot \operatorname{EigenEmbed} \label{eq:eigen_property} \\
    &\geq \operatorname{EigenEmbed}. \label{eq:sample_property}
\end{align}

Eq.~\eqref{eq:eigen_property} and Eq.~\eqref{eq:sample_property} uses the fact that $\operatorname{EigenEmbed} \in [0,1]$ (step 1) and $N > 1$, respectively.
\end{proof}
\paragraph{Equality conditions}  
Equality in Eq.~\eqref{eq:l1-l2_origin_app} occurs if and only if all centered embeddings $\mathbf{v}_i=\mathbf{0} \ \forall i$, this corresponds to all original embeddings $\mathbf{u}_i$ being identical. The final inequality $\sqrt{N \cdot \operatorname{EigenEmbed}} \geq \operatorname{EigenEmbed}$ is strict for $N>1$, which ensures that $\operatorname{RDS} > \operatorname{EigenEmbed}$ unless the embeddings are identical (EigenEmbed = 0). Intuitively, the gap between RDS and EigenEmbed grows as the mean embedding $\bar{\mathbf{u}} \to \mathbf{0}$, reflecting greater dispersion among the embeddings.

\subsection{Proof of Proposition \ref{prop:cosine_zero_centroid}}\label{app:proof_zero_centroid}

\begin{proof}
Compute the squared norm of the total sum:
\[
\Big\|\sum_{i=1}^N \mathbf{u}_i\Big\|_2^2
= \sum_{i=1}^N \sum_{j=1}^N \mathbf{u}_i^\top \mathbf{u}_j.
\]
If \(\bar{\mathbf{u}} = \mathbf{0}\), then \(\sum_i \mathbf{u}_i = \mathbf{0}\) and the left-hand side is zero. Expanding the double sum separates diagonal and off-diagonal terms:
\[
0 = \sum_{i=1}^N \|\mathbf{u}_i\|_2^2 + \sum_{i\neq j} \mathbf{u}_i^\top \mathbf{u}_j
= N + \sum_{i\neq j} \mathbf{u}_i^\top \mathbf{u}_j.
\]
Hence \(\sum_{i\neq j} \mathbf{u}_i^\top \mathbf{u}_j = -N\), and dividing by the \(N(N-1)\) off-diagonal entries gives the stated average:
\[
\frac{1}{N(N-1)}\sum_{i\neq j} \mathbf{u}_i^\top \mathbf{u}_j = -\frac{1}{N-1} < 0.
\]

If for some fixed \(i\) we had \(\mathbf{u}_i^\top \mathbf{u}_j \ge 0\) for all \(j\neq i\), then summing would give
\[
\mathbf{u}_i^\top\Big(\sum_{j\neq i}\mathbf{u}_j\Big) \ge 0.
\]
But \(\sum_{j\neq i}\mathbf{u}_j = -\mathbf{u}_i\) (since the total sum is zero), so the left-hand side equals \(-\|\mathbf{u}_i\|_2^2 = -1\), a contradiction. Thus every \(i\) has at least one \(j\) with \(\mathbf{u}_i^\top \mathbf{u}_j<0\).
\end{proof}

\subsection{Extended Results: Number Of Samples}\label{app:sample-details}


\begin{table}[h!]
\centering
\caption{Detailed performance on GPQA and SVAMP using Falcon3-7B across different number of sampling responses.}
\label{tab:hallucination-full}
\begin{tabular}{c|ccccccccc}
\toprule
$N$ & PRO & SE & Deg & SD & ES & EE & SC & RDS & RDS$_w$ \\
\midrule
\multicolumn{10}{c}{\textbf{GPQA}} \\
\midrule
5 & 0.547 & 0.513 & 0.638 & 0.636 & 0.606 & 0.604 & 0.528 & 0.602 & 0.586 \\
10 & 0.602 & 0.504 & 0.663 & 0.653 & 0.649 & 0.662 & 0.516 & 0.675 & 0.670 \\
20 & 0.612 & 0.531 & 0.656 & 0.647 & 0.653 & 0.643 & 0.490 & 0.675 & 0.672 \\
40 & 0.603 & 0.525 & 0.650 & 0.636 & 0.669 & 0.636 & 0.500 & 0.665 & 0.669 \\
\midrule
\multicolumn{10}{c}{\textbf{SVAMP}} \\
\midrule
5 & 0.918 & 0.845 & 0.871 & 0.882 & 0.636 & 0.877 & 0.887 & 0.866 & 0.876 \\
10 & 0.923 & 0.892 & 0.909 & 0.930 & 0.666 & 0.942 & 0.944 & 0.947 & 0.951 \\
20 & 0.921 & 0.862 & 0.876 & 0.895 & 0.693 & 0.913 & 0.919 & 0.913 & 0.918 \\
40 & 0.920 & 0.817 & 0.848 & 0.884 & 0.693 & 0.925 & 0.941 & 0.934 & 0.934 \\
\bottomrule
\end{tabular}
\end{table}


\subsection{Model And Data Appendix} \label{app:data}
We list the links to the LLMs and datasets in Table~\ref{tab:data}. 

\begin{table*}[h]
\centering
\resizebox{\textwidth}{!}{
\begin{tabular}{l|l}
\midrule
\textbf{Models/Datasets} & \textbf{URL} \\
\midrule
Falcon3-7B & \url{https://huggingface.co/tiiuae/Falcon3-7B-Instruct} \\
Gemma2-9B & \url{https://huggingface.co/google/gemma-2-9b-it}\\
Llama3.1-8B & \url{https://huggingface.co/meta-llama/Llama-3.1-8B-Instruct} \\
Llama3.2-3B & \url{https://huggingface.co/meta-llama/Llama-3.2-3B-Instruct}\\
SciQ & \url{https://github.com/launchnlp/LitCab/blob/main/sciq/test.txt} \\
GPQA & \url{https://huggingface.co/datasets/Idavidrein/gpqa} \\
Arithematics & \url{https://huggingface.co/datasets/EleutherAI/arithmetic/resolve/main/data/single_digit_three_ops.jsonl} \\
SVAMP & \url{https://huggingface.co/datasets/ChilleD/SVAMP} \\
\bottomrule
\end{tabular}}
\caption{Models and Datasets Details.}
\label{tab:data}
\end{table*}

\end{document}